\documentclass[letterpaper,journal]{IEEEtran}
\usepackage{amsmath,amsfonts,amssymb}
\usepackage{amsthm}
\usepackage{algorithmic}
\usepackage{algorithm}
\usepackage{array}
\usepackage[caption=false,font=normalsize,labelfont=sf,textfont=sf]{subfig}
\usepackage{textcomp}
\usepackage{stfloats}
\usepackage{url}
\usepackage{graphicx}
\usepackage{cite}
\usepackage{bm}
\usepackage{booktabs}
\usepackage{mathtools}
\usepackage{placeins}
\usepackage{siunitx}

\theoremstyle{definition}
\newtheorem{definition}{Definition}
\newtheorem{remark}{Remark}
\theoremstyle{plain}
\newtheorem{theorem}{Theorem}
\newtheorem{corollary}{Corollary}
\newtheorem{proposition}{Proposition}

\newcommand{\R}{\mathbb{R}}
\newcommand{\norm}[1]{\left\lVert#1\right\rVert}

\newcommand{\sdf}{\phi}          

\newcommand{\udes}{\mathbf{u}_{\mathrm{des}}}
\newcommand{\usafe}{\mathbf{u}_{\mathrm{safe}}}

\begin{document}

\title{A Closed-Form Dual-Barrier CBF Safety Filter for Holonomic Robots on Incrementally Built Occupancy Grid Maps}

\author{Himanshu Paudel, Basanta Joshi, Dhirendra Raj Madai, Alina Bartaula, Biman Rimal, Sanjay Neupane}




\maketitle

\begin{abstract}
We present a dual-barrier control barrier function (CBF) safety filter for real-time, safety-critical velocity control of holonomic robots operating in incrementally built occupancy-grid maps. As a robot explores an unknown environment, its map grows incrementally and unmapped regions carry irreducible uncertainty: obstacle geometry beyond the explored frontier is entirely unknown, making incursion into unexplored space a source of unquantifiable collision risk especially in cases with only front facing sensor based maps. To address this, we enforce two complementary safety constraints simultaneously: the robot must avoid mapped obstacles and must not venture beyond the explored frontier into regions where no obstacle information exists. Each constraint is derived analytically from the occupancy grid's signed distance field, yielding a closed-form safety filter that requires at most a small linear system solve per control cycle rather than a general-purpose optimizer. On resource-constrained companion computers such as the Raspberry Pi, perception and planning pipelines including visual-inertial SLAM already impose substantial computational load, leaving little headroom for a safety layer; the analytical nature of our filter ensures minimal overhead , leaving the remaining compute budget available for SLAM and planning. An adaptive gain schedule further relaxes the frontier constraint in information-rich regions and tightens it in well-mapped areas, improving exploration efficiency without sacrificing safety guarantees. Because the filter acts purely in velocity space as a minimally invasive correction, it composes transparently with arbitrary nominal controllers, including learning-based navigation frameworks, providing formal safety guarantees without requiring access to or modification of the underlying policy. Hardware flight experiments on a PX4-controlled quadrotor demonstrate zero obstacle contacts across multiple indoor runs, with the filter intervening only when the nominal velocity would violate a safety margin.
\end{abstract}


\section{Introduction}
\IEEEPARstart{S}{afety-critical} velocity control of holonomic robots
operating in incrementally built maps requires simultaneously avoiding
mapped obstacles and respecting the boundary of explored space, where
obstacle information is absent.
These two objectives are in direct tension: aggressive navigation toward
unexplored regions naturally drives the robot toward unmapped space, where
the risk of collision is highest.

Classical reactive controllers, such as artificial potential fields
(APF)~\cite{khatib1986apf}, resolve this tension heuristically: repulsive
potentials slow the robot near obstacles while an attractive potential pulls
it toward the navigation goal.
However, APF controllers provide no formal safety certificate and are
susceptible to local minima.
Model predictive control (MPC) can incorporate safety constraints explicitly,
but its computational cost is prohibitive for onboard execution at high rates
on resource-constrained companion computers.
The safety filter presented here is demonstrated on a UAV platform with an
RRT*-APF nominal controller, but applies directly to any velocity-controlled
holonomic robot; a particularly attractive use case is safety-wrapping
learning-based navigation frameworks.

\emph{Control barrier functions} (CBFs)~\cite{ames2019cbf} offer a principled
middle ground.
A CBF $h:\mathcal{X}\rightarrow\mathbb{R}$ certifies forward invariance of
the safe set $\mathcal{C}=\{x\mid h(x)\geq 0\}$ by enforcing
$\dot{h}(x,u) \geq -\gamma h(x)$, which reduces to a linear constraint on the
control input $u$ and can be resolved as a QP at each timestep.
Recent work has extended CBFs to occupancy-grid maps by constructing barrier
functions from signed distance fields (SDFs)~\cite{ogmcbf2024}, enabling
real-time safety filtering on maps produced by onboard SLAM.

\textbf{Contribution.}
We extend the OGM-CBF framework of~\cite{ogmcbf2024} in three directions:
\begin{enumerate}
  \item \textbf{Dual barrier with admissible shaping class:} we generalise
        the single-barrier OGM-CBF to a \emph{dual} formulation by defining
        two simultaneous barrier functions---one penalising proximity to
        obstacles ($h_1$) and one penalising excursion into unexplored space
        ($h_2$)---each constructed from an SDF composed with a shaping
        function from an admissible class $\mathcal{T}$ (monotone,
        Lipschitz-continuous, sign-preserving, bounded).
        The class-level analysis subsumes the tanh construction as one
        concrete instantiation and clarifies which properties are essential
        for the CBF conditions to hold.
  \item \textbf{Closed-form dual projection:} instead of a general-purpose QP
        solver, we enumerate the four KKT active-set cases analytically,
        reducing the per-cycle computation to at most a $2\times2$ linear
        system.
        A soft-QP bisection fallback handles the singular case when both
        SDF gradient directions coincide, finding the least-violating safe
        velocity.
  \item \textbf{Adaptive $\gamma_2$:} the frontier barrier gain is scheduled
        online using a local uncertainty density estimate derived from the
        occupancy grid.
        This relaxes the frontier constraint in information-rich regions and
        tightens it once the neighbourhood is fully mapped, improving
        navigation efficiency without sacrificing safety.
\end{enumerate}
We validate the complete system with hardware experiments on a
CUAV Nano 7, Raspberry~Pi~4B quadrotor running PX4~v1.16 and ROS~2~Jazzy,
reporting CBF parameters across multiple runs.
The remainder of this paper is structured as follows.
Section~\ref{sec:related} reviews related work.
Section~\ref{sec:problem} states the problem formulation.
Section~\ref{sec:method} presents the CBF filter in detail.
Section~\ref{sec:system} describes the system integration and hardware platform.
Section~\ref{sec:experiments} describes the hardware setup and experimental
results.
Section~\ref{sec:conclusion} concludes.

\section{Related Work}
\label{sec:related}

\subsection{Control Barrier Functions for Robotics}
CBFs were formalised by Ames~\textit{et al.}~\cite{ames2019cbf} and have since
been applied to wheeled robots~\cite{wang2017cbf_multi}, manipulators, and
aerial vehicles~\cite{Wu}.
For a control-affine system $\dot{x}=f(x)+g(x)u$, a control barrier function
$h:\mathcal{X}\rightarrow\mathbb{R}$ defines the safe set
$\mathcal{C}=\{x\in\mathcal{X}\mid h(x)\geq 0\}$.
Forward invariance is enforced by requiring the existence of a class-$\mathcal{K}$
function $\alpha$ such that
\begin{equation}
L_f h(x)+L_g h(x)u \geq -\alpha\bigl(h(x)\bigr),
\end{equation}
which, for the common linear choice $\alpha(h)=\gamma h$ with $\gamma>0$,
reduces to the affine inequality $L_f h(x)+L_g h(x)u \geq -\gamma h(x)$.
The standard formulation appends a CBF constraint to a CLF-based QP
(CLF-CBF-QP), yielding a safety filter that minimally modifies a nominal
controller.
Extensions include high-relative-degree systems~\cite{nguyen2016exponential},
stochastic settings~\cite{clark2021control}, and learning-based
approaches~\cite{srinivasan2020synthesis}.

\subsection{Safety Filtering for UAVs}
Safety-critical control for quadrotors has been studied using
CBFs~\cite{Wu} and Hamilton--Jacobi reachability.
Most existing approaches assume a known, static environment represented by a
geometric model (sphere trees, polytopes), which limits applicability to online
SLAM settings where the map evolves continuously.

\subsection{Occupancy-Grid-Based CBFs}
The closest prior work to ours is the OGM-CBF framework~\cite{ogmcbf2024},
which constructs a single barrier function from the signed distance field of an
occupancy grid and derives the corresponding CBF condition using the Eikonal
property $\norm{\nabla\sdf}=1$.
Our work extends this to a \emph{dual} barrier that additionally constrains
the frontier SDF, introduces a closed-form solver for the resulting
two-constraint QP, and demonstrates the approach on real UAV hardware.

\subsection{Frontier-Based Exploration}
Frontier exploration~\cite{yamauchi1997frontier} identifies boundaries between
known-free and unknown space as candidate goals.
Scoring functions based on information gain~\cite{bircher2016receding},
travel cost, and approach-corridor quality have been proposed to select among
candidate frontiers.

\section{Problem Formulation}
\label{sec:problem}

\subsection{System Model}
Consider a quadrotor operating in a 2D horizontal plane at fixed altitude,
with state $\mathbf{x} = [p_x, p_y]^\top \in \R^2$ and velocity control input
$\mathbf{u} = [v_x, v_y]^\top \in \R^2$.
The kinematic model is:
\begin{equation}
  \dot{\mathbf{x}} = \mathbf{u}.
  \label{eq:kinematics}
\end{equation}
A nominal controller (here, an artificial potential field) produces a desired
velocity $\udes \in \R^2$.
Our goal is to compute a minimally modified safe velocity $\usafe$ such that
two safety constraints are satisfied simultaneously.

\subsection{Occupancy Grid and Signed Distance Fields}
Let $\mathcal{M} \subset \mathbb{Z}^2$ denote an occupancy grid map with
resolution $r$ [m/cell], where each cell takes values in
$\{-1, 0, 100\}$ representing unknown, free, and occupied, respectively.
We define two signed distance fields over $\mathcal{M}$:

\noindent\textbf{Obstacle SDF} $\sdf_{\mathrm{obs}}$: the signed Euclidean
distance to the nearest obstacle cell,
\begin{equation}
  \sdf_{\mathrm{obs}}(\mathbf{x}) > 0 \;\text{ in free space,} \quad
  \sdf_{\mathrm{obs}}(\mathbf{x}) < 0 \;\text{ inside obstacles.}
\end{equation}
Both SDFs satisfy the Eikonal equation $\norm{\nabla\sdf}=1$ almost everywhere
(by Rademacher's theorem applied to the distance transform), which yields
unit-norm gradients suitable for CBF constraint construction.~\cite{Luo2019}

\noindent\textbf{Frontier SDF} $\sdf_{\mathrm{unk}}$: the signed Euclidean
distance to the nearest significant unknown-cell cluster.
Unknown clusters with fewer than $N_{\min}$ cells are discarded to suppress
SLAM noise.
When no significant frontier exists (map fully explored),
$\sdf_{\mathrm{unk}}$ is undefined and the frontier constraint is dropped.

\subsection{Safety Constraints}
We require the vehicle to remain in the \emph{dual-safe set}:
\begin{align}
  \mathcal{C}_1 &= \left\{\mathbf{x} \;\middle|\; \sdf_{\mathrm{obs}}(\mathbf{x}) \geq d_{\mathrm{safe}}\right\}, \\
  \mathcal{C}_2 &= \left\{\mathbf{x} \;\middle|\; \sdf_{\mathrm{unk}}(\mathbf{x}) \geq d_{\mathrm{stop}}\right\},
\end{align}
where $d_{\mathrm{safe}}$ and $d_{\mathrm{stop}}$ are user-defined standoff
distances.
$\mathcal{C}_1$ keeps the drone away from obstacles;
$\mathcal{C}_2$ prevents it from entering unmapped territory where no obstacle
information is available. Fig.~\ref{fig:dstop} illustrates the frontier
standoff distance and obstacle standoff distance, each defined as the sum of the maximum robot radius and a desired clearance from frontier and obstacles respectively.

\begin{figure}
\centering
\vspace{-1mm}
\includegraphics[width=0.68\columnwidth]{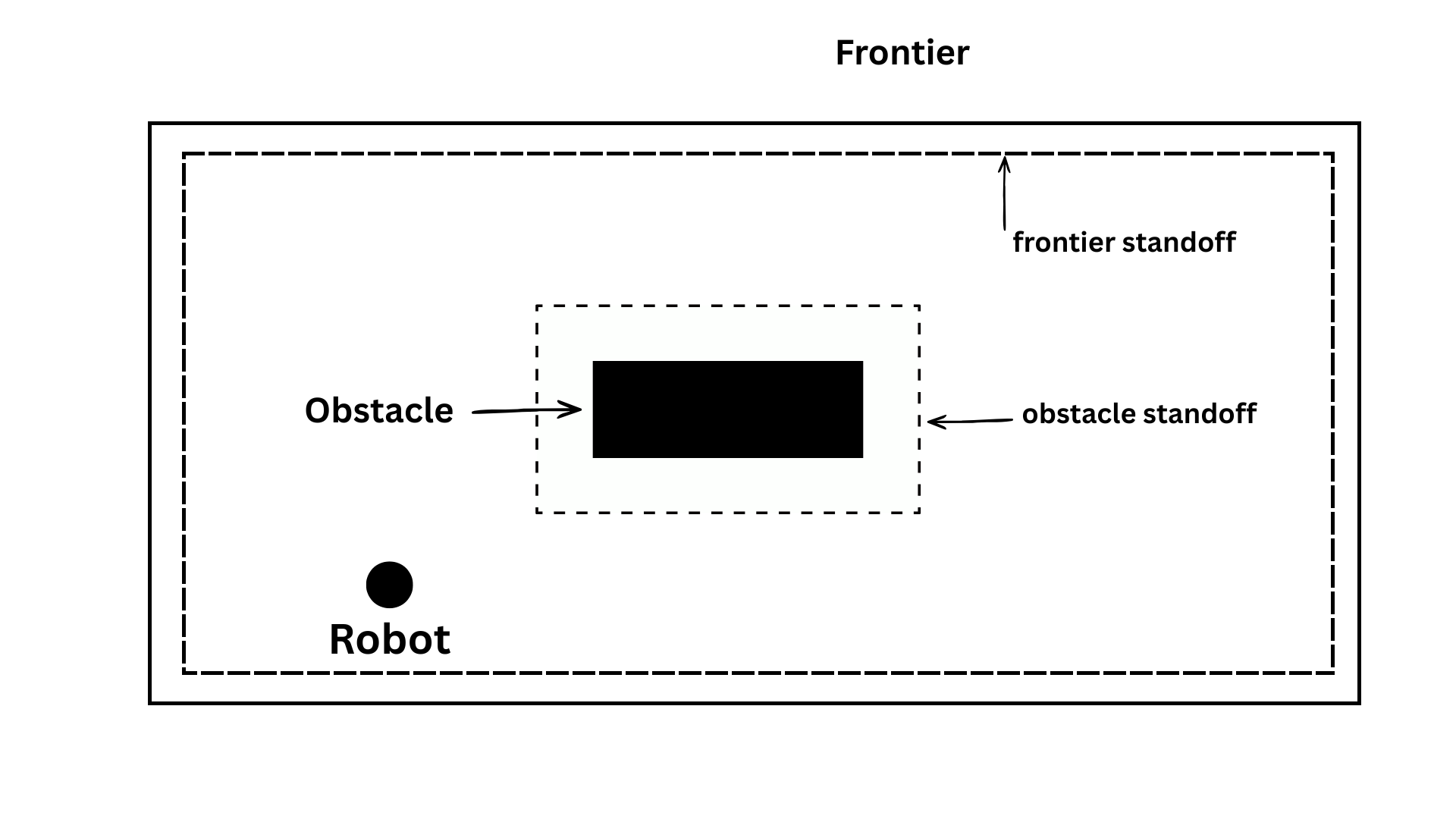}
\caption{Illustration of the frontier and obstacle stand-off distances used in the safety constraints.}
\label{fig:dstop}
\vspace{-2mm}
\end{figure}

\section{Dual-OGM-CBF Safety Filter}
\label{sec:method}

\subsection{Admissible Shaping Function Class}
\label{sec:class}

Following Raja~\textit{et al.}~\cite{ogmcbf2024}, we do not use the raw SDF
as a barrier function directly.
Instead, we compose it with a \emph{shaping function} drawn from the following
admissible class.

\begin{definition}[Admissible shaping class $\mathcal{T}$]
\label{def:class}
A function $T:\R\to\R$ belongs to $\mathcal{T}$ if it satisfies:
\begin{enumerate}
  \item[\textbf{(T1)}] $T(0)=0$ and $T(s) \geq 0 \iff s \geq 0$.
  \item[\textbf{(T2)}] $T$ is strictly increasing: $s_1 < s_2 \Rightarrow T(s_1)<T(s_2)$.
  \item[\textbf{(T3)}] $T$ is continuously differentiable with $|T'(s)|\leq L$
        for some finite $L = L(T) > 0$.
  \item[\textbf{(T4)}] $T$ is bounded: $\sup_{s}|T(s)| < \infty$.
\end{enumerate}
\end{definition}

Property (T1) ensures $T(0)=0$ and sign-preservation, so the zero-superlevel
set of $T(\sdf(\cdot)-d)$ coincides with $\{\sdf \geq d\}$. Shifted sigmoids satisfying $T(0)=0$ are admissible.
Property (T3) guarantees that the Eikonal property $\norm{\nabla\sdf}=1$ (valid
a.e.\ by Rademacher's theorem) carries through to give a bounded gradient norm
for the composed barrier.
Property (T4) ensures $h_i$ is bounded, preventing $b_i = -\gamma_i h_i$ from
becoming arbitrarily negative far from the surface. Without T4, the
constraint $\mathbf{g}_i^\top \mathbf{u} \geq b_i$ would be trivially
satisfiable for any $\mathbf{u}$ at large $\varphi_i$, providing no meaningful
safety guarantee in the interior of $\mathcal{C}_i$.

\begin{remark}
Any $T\in\mathcal{T}$ admits a first-order CBF construction under
Definition~\ref{def:class}.
In particular, $\mathcal{T}$ includes:
(i)~$T(s) = \tanh(as)$, $a>0$ (our instantiation);
(ii)~$T(s) = s/(1+|s|)$;
(iii)~$T(s) = \mathrm{erf}(as)$, $a>0$;
(iv)~any odd sigmoid with bounded derivative.
The analysis in Sections~\ref{sec:barriers}--\ref{sec:soft} holds for all
$T\in\mathcal{T}$; the specific closed forms in
\eqref{eq:hdot_concrete}--\eqref{eq:rhs_concrete} specialise to $T=\tanh$.
\end{remark}

\subsection{Dual Barrier Function Construction}
\label{sec:barriers}

Given $T\in\mathcal{T}$ and a shift parameter $a_i>0$, define the
\emph{scaled shaping function} $T_{a_i}(s)\coloneqq T(a_i s)$. Note that
$T_{a_i}$ satisfies T1, T2, T4 by construction. It satisfies T3 with
inherited constant $L_i = a_i L(T)$, since by the chain rule
$|T'_{a_i}(s)| = a_i|T'(a_i s)| \leq a_i L(T)$. Hence $T_{a_i} \in \mathcal{T}$.
The constraint normal $\mathbf{g}_i$ therefore has norm
$\|\mathbf{g}_i\| = T'_{a_i}(\varphi_i - d_i)\|\nabla\varphi_i\| \leq a_i L(T)$,
which is finite and bounded away from the saturated regime.
The $i$-th barrier function is:
\begin{equation}
  h_i(\mathbf{x}) = T_{a_i}\!\bigl(\sdf_i(\mathbf{x}) - d_i\bigr),
  \quad i \in \{1,2\},
  \label{eq:barrier_general}
\end{equation}
where $(\sdf_1, d_1) = (\sdf_{\mathrm{obs}}, d_{\mathrm{safe}})$ and
$(\sdf_2, d_2) = (\sdf_{\mathrm{unk}}, d_{\mathrm{stop}})$.

\begin{proposition}
\label{prop:safe_set}
For any $T\in\mathcal{T}$, the zero-superlevel set of $h_i$ satisfies
$\{h_i \geq 0\} = \{\sdf_i \geq d_i\} = \mathcal{C}_i$.
\end{proposition}
\begin{proof}
Since $T$ satisfies sign-preservation and strict monotonicity (T1, T2), $T(s)\geq 0 \iff s\geq 0$.
Therefore $h_i(\mathbf{x}) = T_{a_i}(\sdf_i - d_i) \geq 0 \iff \sdf_i(\mathbf{x})\geq d_i$.
\end{proof}

\begin{proposition}[Forward invariance of $\mathcal{C}_1 \cap \mathcal{C}_2$]
\label{prop:forward_invariance}
Let $h_1, h_2$ be constructed as in~\eqref{eq:barrier_general} for any
$T \in \mathcal{T}$. If $\usafe$ is the solution to~\eqref{eq:qp}, then the
dual-safe set $\mathcal{C}_1 \cap \mathcal{C}_2$ is forward invariant under the
kinematic model~\eqref{eq:kinematics}.
\end{proposition}
\begin{proof}
It suffices to show that $\usafe$ satisfies both halfspace constraints
simultaneously, i.e., $\mathbf{g}_i^\top \usafe \geq b_i$ for $i=1,2$, which
implies $\dot{h}_i \geq -\gamma_i h_i$ for each $i$, and hence $h_i(t) \geq 0$
for all $t \geq 0$ whenever $h_i(0) \geq 0$. The QP~\eqref{eq:qp} is a
projection onto $H_1 \cap H_2$, where $H_i = \{\mathbf{u} : \mathbf{g}_i^\top
\mathbf{u} \geq b_i\}$. Each $H_i$ is a closed halfspace, so $H_1 \cap H_2$ is a
closed convex set. The KKT enumeration of Section~\ref{sec:dual_projection}
finds the unique projection of $\udes$ onto $H_1 \cap H_2$ whenever
$\det(G) \neq 0$, where $G$ is the Gram matrix defined in Section~\ref{sec:dual_projection}. When
$\det(G) = 0$ ($\mathbf{g}_1 \parallel \mathbf{g}_2$) and $H_1 \cap H_2 = \emptyset$,
the soft-QP fallback of Section~\ref{sec:soft} finds the least-violating
$\usafe$ with minimum slack $\delta^\star$; in this case both barriers decay at a
bounded rate governed by $\delta^\star$. In all non-degenerate cases,
$\usafe \in H_1 \cap H_2$ by construction, giving
$\dot{h}_i(\mathbf{x}, \usafe) \geq -\gamma_i h_i(\mathbf{x})$ for $i=1,2$. By
the standard CBF invariance theorem~\cite{ames2019cbf}, each $\mathcal{C}_i$ is
forward invariant, and therefore so is $\mathcal{C}_1 \cap \mathcal{C}_2$.
\end{proof}

\begin{theorem}
For any $T\in\mathcal{T}$, the function
$h_i(\mathbf{x}) = T_{a_i}(\varphi_i(\mathbf{x}) - d_i)$ is a valid CBF for
$\mathcal{C}_i$ with decay rate $\gamma_i$, yielding the halfspace
constraint~\eqref{eq:cbf_constraint_general}.
\end{theorem}
\begin{proof}
$\dot{h}_i = T'_{a_i}(\varphi_i - d_i)\nabla\varphi_i^\top \mathbf{u}$. Since
$T'_{a_i} > 0$ (T2) and $\|\nabla\varphi_i\| = 1$ a.e.\ (Eikonal),
$\dot{h}_i$ is linear and nondegenerate in $\mathbf{u}$. Setting
$\dot{h}_i \geq -\gamma_i h_i$ gives
constraint~\eqref{eq:cbf_constraint_general}. \qedhere
\end{proof}

\begin{corollary}
For $T = \tanh$, constraint~\eqref{eq:cbf_constraint_general} specialises to
\eqref{eq:hdot_concrete}--\eqref{eq:rhs_concrete}.
\end{corollary}

\textbf{Concrete instantiation.}
We choose $T = \tanh$, giving:
\begin{equation}
  h_i(\mathbf{x}) = \tanh\!\bigl(a_i(\sdf_i(\mathbf{x}) - d_i)\bigr)
                 \in (-1,\,1).
  \label{eq:barrier_tanh}
\end{equation}
This choice is convenient because $\tanh$ has an explicit derivative
$\tanh'(s) = 1 - \tanh^2(s) = \mathrm{sech}^2(s)$, yielding compact
closed-form expressions below, and is numerically stable for large $|\sdf_i|$.

\subsection{CBF Conditions and Linear Velocity Constraints}
\label{sec:constraints}

For a general $T\in\mathcal{T}$, the time derivative of $h_i$ along the
kinematic model~\eqref{eq:kinematics} is:
\begin{equation}
  \dot{h}_i = T_{a_i}'\!\bigl(\sdf_i - d_i\bigr)\,\nabla\sdf_i^\top \mathbf{u}.
  \label{eq:hdot_general}
\end{equation}
Since $T_{a_i}' > 0$ (strict monotonicity, T2) and $\|\nabla\varphi_i\| = 1$
a.e.\ (Eikonal), $\dot{h}_i$ is linear in $\mathbf{u}$ with a well-defined,
bounded coefficient. On the discretised occupancy grid, $\nabla\varphi_i$ is
approximated by central finite differences and is well-defined at every cell;
the CBF condition $\dot{h}_i \geq -\gamma_i h_i$ therefore holds pointwise
without qualification.
The CBF condition $\dot{h}_i \geq -\gamma_i h_i$ becomes the halfspace
constraint:
\begin{equation}
  \underbrace{T_{a_i}'\!\bigl(\sdf_i-d_i\bigr)\,\nabla\sdf_i^\top}_{\mathbf{g}_i^\top}
  \mathbf{u}
  \;\geq\;
  \underbrace{-\gamma_i\,T_{a_i}(\sdf_i-d_i)}_{\eqqcolon\, b_i}.
  \label{eq:cbf_constraint_general}
\end{equation}

\textbf{Specialisation to $T=\tanh$.}
Using $\tanh'(s)=\mathrm{sech}^2(s) = 1-\tanh^2(s)$:
\begin{align}
  \dot{h}_i &= a_i\bigl(1-h_i^2\bigr)\,\nabla\sdf_i^\top\mathbf{u},
  \label{eq:hdot_concrete} \\
  b_i       &= -\gamma_i h_i
             = -\gamma_i\tanh\!\bigl(a_i(\sdf_i-d_i)\bigr).
  \label{eq:rhs_concrete}
\end{align}
Because $\norm{\nabla\sdf_i}=1$ a.e., the constraint normal
$\mathbf{g}_i = a_i(1-h_i^2)\nabla\sdf_i$ has norm $a_i(1-h_i^2)\leq a_i$,
so the halfspace is always well-conditioned away from the saturated regime
$h_i\to\pm 1$.

\subsection{Closed-Form Dual Projection}
\label{sec:dual_projection}

Given the linear constraints~\eqref{eq:cbf_constraint_general}, the safety
filter solves the minimal-modification QP:
\begin{equation}
  \usafe = \arg\min_{\mathbf{u}} \tfrac{1}{2}\norm{\mathbf{u} - \udes}^2
  \quad\text{s.t.}\quad
  \mathbf{g}_i^\top \mathbf{u} \geq b_i,\; i=1,2.
  \label{eq:qp}
\end{equation}
This derivation is valid for \emph{any} $T\in\mathcal{T}$; the specific values
of $\mathbf{g}_i$ and $b_i$ depend on the chosen instantiation.
This is a projection onto the intersection of two halfspaces in $\R^2$.
We enumerate the four KKT active-set cases analytically:

\noindent\textbf{Case 1} (neither constraint violated): $\usafe = \udes$.

\noindent\textbf{Case 2} (only constraint $i$ violated): project onto the
single halfspace,
\begin{equation}
  \usafe = \udes + \frac{b_i - \mathbf{g}_i^\top\udes}{\norm{\mathbf{g}_i}^2}\,\mathbf{g}_i,
  \label{eq:single_proj}
\end{equation}
then verify the other constraint.  If satisfied, accept; otherwise fall
through to Case~4.

\noindent\textbf{Case 3}: symmetric to Case~2 for the other constraint.

\noindent\textbf{Case 4} (both constraints active): solve the $2\times2$ Gram
system
\begin{equation}
  \begin{bmatrix} \mathbf{g}_1^\top\mathbf{g}_1 & \mathbf{g}_1^\top\mathbf{g}_2 \\
                  \mathbf{g}_2^\top\mathbf{g}_1 & \mathbf{g}_2^\top\mathbf{g}_2 \end{bmatrix}
  \begin{bmatrix}\lambda_1\\\lambda_2\end{bmatrix}
  =
  \begin{bmatrix}b_1 - \mathbf{g}_1^\top\udes\\b_2 - \mathbf{g}_2^\top\udes\end{bmatrix},
  \label{eq:gram}
\end{equation}
yielding $\usafe = \udes + \lambda_1\mathbf{g}_1 + \lambda_2\mathbf{g}_2$.
The determinant of the Gram matrix is
$\det = \norm{\mathbf{g}_1}^2\norm{\mathbf{g}_2}^2 - (\mathbf{g}_1^\top\mathbf{g}_2)^2
       = \sin^2\theta\,\norm{\mathbf{g}_1}^2\norm{\mathbf{g}_2}^2$,
which vanishes if and only if $\mathbf{g}_1 \parallel \mathbf{g}_2$
(obstacle and frontier gradients aligned).
KKT dual feasibility requires $\lambda_1, \lambda_2 \geq 0$; if one multiplier
is negative, the corresponding constraint is not truly active and we fall back
to the single-constraint projection.

\subsection{Soft-QP Fallback}
\label{sec:soft}

When $\mathbf{g}_1 \parallel \mathbf{g}_2$ and the hard problem is infeasible,
we solve a soft relaxation:
\begin{equation}
  \min_{\mathbf{u},\delta\geq 0}
    \tfrac{1}{2}\norm{\mathbf{u}-\udes}^2 + \tfrac{p}{2}\delta^2
  \quad\text{s.t.}\quad
  \mathbf{g}_i^\top\mathbf{u} \geq b_i - \delta,\; i=1,2.
  \label{eq:soft_qp}
\end{equation}
We solve~\eqref{eq:soft_qp} for $\delta^\star$ via bisection (20 iterations)
and then apply the dual projection of Section~\ref{sec:dual_projection} to the
relaxed constraints.

\subsection{Adaptive Frontier Gain $\gamma_2$}
\label{sec:adaptive}

A fixed $\gamma_2$ imposes the same frontier-avoidance stiffness regardless of
how much information is available nearby.
We instead compute a local uncertainty density:
\begin{equation}
  \rho(\mathbf{x}) =
    \frac{|\{\text{unknown cells within sensor disc}\}|}{|\text{disc area in cells}|} \in [0,1],
  \label{eq:rho}
\end{equation}
and schedule:
\begin{equation}
  \gamma_2 = \gamma_{\min} + (\gamma_{\max} - \gamma_{\min})(1 - \rho).
  \label{eq:gamma2}
\end{equation}
When $\rho$ is high (many unknown cells nearby), $\gamma_2 \to \gamma_{\min}$,
relaxing the frontier constraint and allowing the drone to approach unexplored
space.
When $\rho \approx 0$ (mostly mapped neighbourhood), $\gamma_2 \to \gamma_{\max}$,
tightening the constraint since there is no exploration value in pushing further.
Since $\gamma_2(\mathbf{x}) \geq \gamma_{\min} > 0$ for all $\mathbf{x}$, the CBF
condition $\dot{h}_2 \geq -\gamma_2(\mathbf{x})h_2$ is sufficient for forward
invariance of $\mathcal{C}_2$ by \cite{ames2019cbf}.

\subsection{Speed Ceiling}

After the dual projection, a ball projection enforces a maximum speed:
\begin{equation}
  \usafe \leftarrow
  \begin{cases}
    \usafe & \text{if } \norm{\usafe} \leq v_{\max}, \\
    v_{\max}\,\usafe/\norm{\usafe} & \text{otherwise.}
  \end{cases}
\end{equation}

\subsection{Complete Filter}

Algorithm~\ref{alg:filter} summarises the complete CBF filter.

\begin{algorithm}[t]
\caption{CBF Safety Filter (general $T\in\mathcal{T}$, tanh instantiation)}\label{alg:filter}
\begin{algorithmic}
\REQUIRE $\udes$, $\sdf_{\mathrm{obs}}$, $\sdf_{\mathrm{unk}}$ (or \textsc{None}),
         grid cell $(g_x, g_y)$, resolution $r$
\STATE Compute $\mathbf{g}_1, b_1$ from $\sdf_{\mathrm{obs}}$ via~\eqref{eq:cbf_constraint_general}--\eqref{eq:rhs_concrete}
\IF{$\sdf_{\mathrm{unk}} = \textsc{None}$}
  \STATE $\usafe \leftarrow$ single projection~\eqref{eq:single_proj} with $(\mathbf{g}_1, b_1)$
\ELSE
  \STATE Compute $\mathbf{g}_2, b_2$ from $\sdf_{\mathrm{unk}}$; update $\gamma_2$ via~\eqref{eq:gamma2}
  \STATE Attempt dual projection (Cases 1--4, Section~\ref{sec:dual_projection})
  \IF{residual constraint violated}
    \STATE $\usafe, \delta^\star \leftarrow$ soft-QP bisection~\eqref{eq:soft_qp}
  \ENDIF
\ENDIF
\STATE Apply speed ceiling; \textbf{return} $\usafe$
\end{algorithmic}
\end{algorithm}

\section{System Integration}
\label{sec:system}

\subsection{Hardware Platform}
Experiments are conducted on a custom quadrotor comprising a CUAV Nano 7 flight
controller running PX4~v1.16 and a Raspberry Pi~4B (4~GB) companion computer.
Perception is provided by Intel Realsense D435i Camera and a TFMini laser rangefinder for altitude hold.
The vehicle communicates via Micro-XRCE DDS over ethernet; the companion computer runs
ROS~2~Jazzy and controls the drone in PX4 Offboard mode by streaming
velocity setpoints at \SI{10}{\hertz}.

\subsection{Software Architecture}
The exploration stack comprises three layers:
\begin{enumerate}
  \item \textbf{Global planner}: frontier detection on the RTAB-Map occupancy
        grid, followed by RRT$^*$ path planning on an inflated costmap.
  \item \textbf{Local controller}: artificial potential field (APF) tracking
        the current RRT$^*$ waypoint, with repulsive terms for mapped obstacles.
  \item \textbf{Safety filter}: TanhDualCBF applied to the APF output velocity
        before publication to \texttt{cmd\_vel}.
\end{enumerate}

\subsection{Parameter Settings}
Table~\ref{tab:params} lists the CBF parameters used in all experiments.

\begin{table}[!t]
\caption{CBF Parameter Settings\label{tab:params}}
\centering
\begin{tabular}{lll}
\hline
Parameter & Symbol & Value \\
\hline
Obstacle standoff      & $d_{\mathrm{safe}}$   & \SI{0.35}{\metre} \\
Frontier standoff      & $d_{\mathrm{stop}}$   & \SI{0.35}{\metre} \\
Obstacle tanh sharpness & $a_1$                & 2.0 \\
Frontier tanh sharpness & $a_2$                & 2.0 \\
Obstacle CBF gain      & $\gamma_1$            & 1.5 \\
Frontier CBF gain (max) & $\gamma_{2,\max}$    & 1.0 \\
Frontier CBF gain (min) & $\gamma_{2,\min}$    & 0.2 \\
Maximum speed          & $v_{\max}$            & \SI{0.20}{\metre\per\second} \\
Min.\ unknown cluster  & $N_{\min}$            & 25 cells \\
\hline
\end{tabular}
\end{table}

\section{Experimental Results}
\label{sec:experiments}

\subsection{Simulation Results}

Simulation experiments were conducted to evaluate the performance of the proposed CBF filter against the RRT*-APF  only baseline. While the baseline controller maintains safety through highly conservative, distance-based repulsion, this heuristic approach inherently restricts the drone's operational envelope, causing it to shy away from high curvature areas and unexplored boundaries. Effectively, the baseline controller prioritizes distance from obstacles over mission progress, leading to a conservative avoidance pattern characterized by premature repulsion from frontier interfaces.

By providing formal safety guarantees, the filter allows the drone to aggressively navigate boundary proximal regions that the baseline controller would otherwise avoid with an instantaneous opposing velocity as summarized in Table II: the CBF-enabled drone achieves a final explored area of 81.68 m², increase over the baseline’s 56.35 m². The average adaptive $\gamma_2$ was in the range of 0.733 - 1.0 . The soft-QP fallback was not triggered.The baseline's higher minimum clearance reflects the APF's highly reactive repulsion, which keeps the drone well away from obstacles at the cost of exploration progress; the CBF filter moderates this repulsion, permitting closer frontier-proximal approach while maintaining physical clearance above the drone radius.

Table~\ref{tab:comparison_full} summarizes the quantitative comparison between
the proposed CBF-enabled approach and the RRT*-APF only baseline.

\begin{table*}[!b]
\centering
\caption{Comparative performance metrics between the proposed CBF-enabled approach and the RRT*-APF only baseline.\label{tab:comparison_full}}
\normalsize
\begin{tabular}{l c c}
\toprule
\textbf{Metric} & \textbf{CBF (Proposed)} & \textbf{Baseline (APF Only)} \\
\midrule
\multicolumn{3}{l}{\textit{Mission Performance}} \\
Total exploration time (s) & 120.0 & 120.0 \\
Explored area (m$^2$) & 81.68 & 56.35 \\
Total path length (m) & 15.88 & 13.35 \\
Average speed (m/s) & 0.133 & 0.112 \\
\midrule

\multicolumn{3}{l}{\textit{Safety \& CBF Metrics}} \\
Global min clearance (m) & 0.304 & 0.412 \\
Obstacle violation ticks & 11 & 0 \\
Frontier violation ticks & 66 & 56 \\
CBF intervention rate & 0.0467 & 0.0000 \\
CBF speed clips & 189 & 0 \\
Avg CBF slack & 0.0012 & 0.0000 \\
Avg adaptive $\gamma_2$ & 0.9253 & --- \\
\bottomrule
\end{tabular}
\end{table*}
\subsection{Hardware Results}

\subsubsection{Experimental Setup}

Experiments were conducted in a hall with relatively open space with some obstacles.
The drone was initialised at different positions and was commanded to explore
autonomously for up to 120~seconds. In practice, individual runs were terminated
early due to low battery or proximity-triggered operator intervention; only data
up to the point of termination is reported for each run.
Three hardware runs are reported.

\subsubsection{Representative Hardware Traces}
Fig.~\ref{fig:hardware_traces} presents representative traces from Run~1,
including the barrier values, raw SDF distances, and the CBF-induced velocity
change. Together with Table~\ref{tab:interventions}, these plots summarize the
same hardware-run category from complementary perspectives: time-series
behavior in the figure and aggregate run-level metrics in the table.

\begin{figure*}[!t]
\centering
\subfloat[CBF barrier functions]{\includegraphics[width=0.325\textwidth]{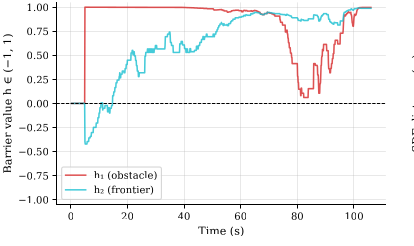}\label{fig:hardware_cbf}}
\hfil
\subfloat[Raw SDF distances]{\includegraphics[width=0.325\textwidth]{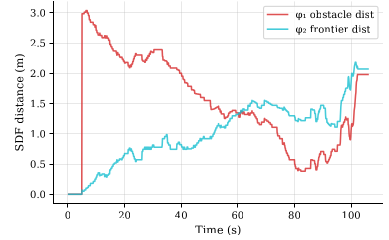}\label{fig:hardware_sdf}}
\hfil
\subfloat[CBF-induced velocity change]{\includegraphics[width=0.325\textwidth]{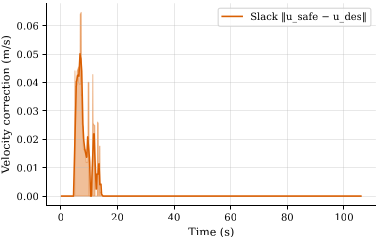}\label{fig:hardware_slack}}
\caption{Representative hardware traces from Run~1 showing the barrier-function evolution,
raw SDF distances, and the change in commanded velocity induced by the CBF
filter.}
\label{fig:hardware_traces}
\end{figure*}

\subsubsection{Intervention Analysis}
Table~\ref{tab:interventions} summarises the CBF-specific metrics across the
three hardware runs. The soft-QP Fallback was not triggered across any runs.

\begin{table*}[!t]
\caption{CBF Metrics Across Hardware Runs\label{tab:interventions}}
\centering
\begin{tabular}{lccccccc}
\hline
Run & Ticks eval. & Obs. active & Front. active & Interv. rate & Infeasible & Speed clips & Min $h_1$ / Min $h_2$ \\
\hline
Run 1 & 1058 & 0 & 71 & 0.0671 & 0 & 56 & 0.0000 / $-0.4219$ \\
Run 2 & 1008 & 0 & 145 & 0.1438 & 0 & 89 & 0.0000 / $-0.9743$ \\
Run 3 & 1198 & 4 & 172 & 0.1469 & 0 & 107 & $-0.1974$ / $-0.7163$ \\
\hline
\end{tabular}
\vspace{1mm}
\begin{tabular}{lcc}
\hline
Run & Avg. CBF slack & Max CBF slack \\
\hline
Run 1 & 0.0018 & 0.0647 \\
Run 2 & 0.0098 & 0.1844 \\
Run 3 & 0.0056 & 0.0926 \\
\hline
\end{tabular}
\end{table*}

\subsubsection{Safety Record}
Across all hardware runs, zero obstacle contacts were recorded. The minimum 
observed clearance of 0.25\,m occurred in Run~3; although the resulting velocity 
correction was small in magnitude, at such proximity even a marginal deflection 
is safety-critical. The barrier $h_1$ briefly dipped below zero in this run, 
however the physical clearance remained above the drone radius of 0.22\,m, 
consistent with the conservative 0.35\,m safety margin encoded in the barrier 
function providing a buffer against such transient violations. 

\section{Discussion}
\label{sec:discussion}

\textbf{Limitations.}
The current implementation operates in 2D (horizontal plane at fixed altitude).
Extension to 3D would require volumetric SDF computation, which is
computationally more demanding.
The frontier barrier relies on the occupancy grid produced by RTAB-Map; SLAM
failures or delayed loop closures can temporarily misrepresent the frontier SDF. Furthermore, the framework currently assumes a static environment; handling dynamic obstacles would necessitate the integration of velocity prediction within the CBF constraints. We also note that the tuning of class-$\mathcal{K}$ functions presents a trade-off between safety and performance, and in highly cluttered scenarios, the underlying QP solver may face infeasibility, requiring more robust slack variable handling.

\section{Conclusion}
\label{sec:conclusion}
We presented a dual-barrier CBF safety filter with a closed form KKT 
solution that simultaneously enforces obstacle standoff and frontier 
containment for autonomous UAV exploration. Hardware validation on a 
resource-constrained PX4 quadrotor confirms zero obstacle contacts 
across all runs, with the analytical filter imposing negligible overhead 
alongside a full SLAM and planning stack.

\vspace{-0.5ex}
\section*{Author Contributions}
Himanshu Paudel conceptualized the theoretical framework, formulated the control barrier function (CBF)-based filter, derived the analytical solutions, and developed the exploration planner. Himanshu Paudel also led simulation efforts, contributed to visual-inertial odometry (VIO), SLAM, and hardware implementation, and wrote the original draft of the manuscript. Basanta Joshi contributed to VIO and SLAM development, as well as simulation, hardware implementation, and experiments. Dhirendra Raj Madai and Alina Bartaula designed the UAV platform, integrated the physical hardware components, and managed electrical wiring and avionics. Biman Rimal and Sanjay Neupane provided project supervision, secured hardware resources, offered high-level structural guidance, and contributed to manuscript review and editing.
\vspace{-1.0ex}



\end{document}